\documentclass[a4paper]{scrartcl}

\usepackage{fullpage}

\usepackage[english]{babel}
\usepackage{amsmath}
\usepackage{amsthm}
\usepackage{todonotes}
\usepackage{comment}
\usepackage[ruled, slide]{algorithm2e}

\usepackage{bbm}

\newcommand{\IN}{\mathbbm{N}}

\newcommand{\IR}{\mathbbm{R}}

\newcommand{\cLc}{\mathcal{L}}
\newcommand{\OPTc}{OPT}

\newcommand{\cL}{\mathcal{L}}
\newcommand{\cN}{\mathcal{N}}
\newcommand{\cC}{\mathcal{C}}
\newcommand{\cM}{\mathcal{M}}

\newcommand{\OPTdiam}{OPT_{diam}}

\newcommand{\bigO}{{\cal O}}
\newcommand{\dpr}[2]{\left\langle#1,#2\right\rangle}

\DeclareMathOperator{\cost}{cost}
\DeclareMathOperator{\argmin}{argmin}
\DeclareMathOperator{\argmax}{argmax}

\renewcommand{\epsilon}{\varepsilon}
\newcommand{\eps}{\epsilon}

\newcommand{\tmu}{\tilde{\mu}}
\newcommand{\tsigma}{\tilde{\sigma}}

\newcommand{\tw}{\tilde{w}}
\newcommand{\ttheta}{\tilde{\theta}}

\newcommand{\OneTo}[1]{\in[#1]}

\usepackage{mathtools}

\DeclarePairedDelimiter\abs{\lvert}{\rvert}%
\DeclarePairedDelimiter\norm{\lVert}{\rVert}%

\newtheorem{definition}{Definition}
\newtheorem{theorem}[definition]{Theorem}
\newtheorem{corollary}[definition]{Corollary}

\newtheorem{lemma}[definition]{Lemma}

\newtheorem{remark}[definition]{Remark}

\newtheorem{problem}[definition]{Problem}

\usepackage{authblk}

\begin{document}
\title{Hard-Clustering with Gaussian Mixture Models}
\author{Johannes Bl\"omer}
\author{Sascha Brauer}
\author{Kathrin Bujna}
\affil
{
  Department of Computer Science \\
  Paderborn University \\
  33102 Paderborn, Germany
}
\maketitle

\begin{abstract}
 Training the parameters of statistical models to describe a given data set is a central task in the field of data mining and machine learning. 
A very popular and powerful way of parameter estimation is the method of maximum likelihood estimation (MLE). 
Among the most widely used families of statistical models are mixture models, especially, mixtures of Gaussian distributions. 

 A popular hard-clustering variant of the MLE problem is the so-called complete-data maximum likelihood estimation (CMLE) method. 
 The standard approach to solve the CMLE problem  is the Classification-Expectation-Maximization (CEM) algorithm \cite{Celeux92}. 
 Unfortunately, it is only guaranteed that the algorithm converges to some (possibly arbitrarily poor) stationary point of the objective function. 
 
 In this paper, we present two algorithms for a restricted version of the CMLE problem. 
 That is, our algorithms approximate reasonable solutions to the CMLE problem which satisfy certain natural properties.
 Moreover, they compute solutions whose cost (i.e. complete-data log-likelihood values) are at most a factor $(1+\epsilon)$ worse than the cost of the solutions that we search for. 
Note the CMLE problem in its most general, i.e. unrestricted, form is not well defined and allows for trivial optimal solutions that can be thought of as degenerated solutions.

\end{abstract}


\clearpage

\section{Preliminaries}

Given set of observations, the objective of the CMLE problem is to find a Gaussian mixture model and a hard clustering with maximum complete-data likelihood.
In this section, we will first describe and define this objective function.
Then, we will present an alternating optimization scheme for this problem.
However, the problem is not well-defined.
Hence, we will restrict the problem to reasonable instances and solutions.

\subsection{Complete-Data Log-Likelihood}
Let $X\subset\IR^d$ be a finite set of observations. 
Given a spherical Gaussian distribution $\cN_d(\mu,\sigma)$, the \emph{likelihood} that all $x\in X$ have been drawn according to $\cN_d(\mu,\sigma)$ is given by 
\[ \prod_{x\in X} \cN_d(x | \mu,\sigma)\ ,\]
assuming that the observations have been drawn independently at random.

\begin{definition}
Given a finite set $X\subset \IR^d$ and a spherical Gaussian distribution with mean $\mu\in\IR^d$ and variance $\sigma^2 \in\IR$, let 
\[ \cL_X(\mu,\sigma^2) 
\coloneqq-\ln \left( \prod_{x\in X} p(x|\mu,\sigma^2)\right) 
=   \frac{\abs{X}d}{2}\ln(2\pi \sigma_k^2) + \frac{1}{2\sigma_k^2} \sum_{x\in X} \norm{x-\mu_k}^2 \ .\]
We denote the minimal value by $\OPTc(X,1) = \min_{(\mu,\sigma^2)} \cL_X(\mu,\sigma^2)$. 
\end{definition}

Now consider a Gaussian mixture model (GMM) given by parameters $\theta = \{(w_k,\mu_k,\sigma^2_k)\}_{k=1}^K$.
Drawing an observation $x_n$ according to a GMM corresponds to a two-step process: 
\begin{enumerate}
  \item Draw a component $z_n\OneTo{K}$ with probability $p(z_n=k|\theta)=w_k$.
  \item Draw an observation $x_n \in X$ according to $\cN_d(\mu_{z_n},\sigma_{z_n})$.
\end{enumerate}
Note that the assignment $z_n\OneTo{K}$ is a (latent) random variable in this two-step process.
With the help of this random variable, we can compute the likelihood that observation $x\in X$ has been generated by the $k$-th component of the GMM, i.e.
\[ p(x_n,z_n=k|\theta) = p(z_n=k|\theta)\cdot p(x_n|z_n=k,\theta) = w_k\cdot \cN_d(x|\mu_k,\sigma_k)\ .\]
Since $x_n$ and $z_n$ completely describe the two-step process, the likelihood $p(x_n,z_n|\theta)$ is also called \emph{complete-data likelihood}, while $p(x_n|\theta) = \sum_{z_n=1}^K p(x_n,z_n|\theta)$ is refered to as \emph{(marginal) likelihood}.

Assume, we are given a set of observations $X=\{x_n\}_{n=1}^N$ and assignments $\{z_n\}_{n=1}^N$.
Then, the likelihood that all observations have been drawn according to a GMM $\theta$ and that each $x_n$ has been generated by the $z_n$-th component, is given by 
\begin{align} 
 \prod_{n=1}^N p(x_n,z_n|\theta) = \prod_{n=1}^N  w_{z_n}\cdot \cN_d(x_n|\mu_{z_n},\sigma_{z_n}) \ ,\label{eq:prelim:complete-data-1}
\end{align}
assuming that the observations have been drawn inpendently at random.
Note that the assignments $\{z_n\}_{n=1}^N$ define a partition $\cC=\dot\cup_{k=1}^K C_k$ via $x_n\in C_k$ iff $z_n=k$. 
Hence, we can also rewrite Equation~\eqref{eq:prelim:complete-data-1} as
\[ \prod_{k=1}^K \prod_{x_n\in C_k} p(x_n,z_n=k|\theta) =  \prod_{k=1}^K \prod_{x_n\in C_k}  w_k\cdot \cN_d(x_n|\mu_k,\sigma_k)\ .  \]
By taking (negative) logarithm of this expression, we obtain
\begin{align*}
 &-\log\left(\prod_{k=1}^K \prod_{x_n\in C_k} p(x_n,z_n=k|\theta)\right) \\
 &= \sum_{k=1}^K \sum_{x_n\in C_k}\left( \ln(w_k) + \ln\left(\cN_d(x_n|\mu_k,\Sigma_k\right)\right) \\
 &= \sum_{k=1}^K  \cLc_{C_k}(\mu_k,\sigma^2_k) - \ln(w_k)\cdot \abs{C_k} \ .
\end{align*}

\begin{definition}
 Given a finite set $X\subset \IR^d$, a partition $\cC = \{C_1,\ldots,C_K\}$ of $X$, and a mixture of spherical Gaussians with parameters $\theta = \{(w_k,\mu_k,\sigma^2_k)\}_{k=1}^K$, 
 we call
 \[ \cLc_X(\theta, \cC) \coloneqq \sum_{k=1}^K  \cLc_{C_k}(\mu_k,\sigma^2_k) - \ln(w_k)\cdot \abs{C_k}  
\]
the complete-data negative log-likelihood. 
\end{definition}

Note that a solution maximizing the complete-data likelihood also minimizes the complete-data negative log-likelihood, and vice versa. 
Therefore, we define the \emph{complete-cata maximum likelihood estimation} (CMLE) problem as follows. 

\begin{problem}[CMLE]\label{prob-cmle}
Given a finite set $X\subset \IR^d$ and an integer $K\in \IN$, find a partition $\cC = \{C_1,\ldots,C_K\}$ of $X$ and a mixture of spherical Gaussians with parameters $\theta = \{(w_k,\mu_k,\sigma^2_k)\}_{k=1}^K$ minimizing $\cLc_X(\theta, \cC)$.
We denote the minimal value by $\OPTc(X,K)$.

For a fixed model $\theta$, we let $\cLc_X(\theta) = \min_{\cC}\cLc_X(\theta,\cC)$. 
Analogously, for a fixed clustering $\cC$, we let $\cLc_X(\cC) = \min_{\theta}\cLc_X(\theta,\cC)$.
\end{problem}

\begin{definition}
Given parameters $(w_k,\mu_k,\sigma_k^2)$ and a cluster $C_k\subseteq X$, we let
\[  \cLc_{x}(w_k,\mu_k,\sigma_k^2) \coloneqq \frac{d}{2}\ln(2\pi \sigma_k^2) + \frac{1}{2\sigma_k^2} \norm{x-\mu_k}^2 - \ln(w_k), \ \]
and 
\[ \cLc_{C_k}(w_k,\mu_k,\sigma_k^2) \coloneqq \sum_{x\in C_k} \cLc_{x}(w_k,\mu_k,\sigma_k^2) \  .\]
\end{definition}

\begin{remark}
 For all partitions $\cC = \{C_1,\ldots,C_K\}$, we have
\[ \cLc_X(\cC) = \sum_{k=1}^K  \OPTc(C_k,1) - \ln\left(\frac{\abs{C_k}}{\abs{X}}\right)\cdot \abs{C_k} \ .\]
 For all $\theta = \{(w_1,\mu_1,\sigma^2_1),\ldots,(w_K,\mu_K,\sigma^2_K)\}$, we have
 \[ \cLc_X(\theta) = \sum_{n=1}^N \argmin_{k\OneTo{K}} \{ \cLc_{x}(w_k,\mu_k,\sigma_k^2) \} \ . \]
\end{remark}

\subsection{Alternating Optimization Scheme (CEM algorithm)}

An \emph{alternating optimization algorithm} for this problem is given by the following first order optimality conditions.
Fixing the partition $\cC=\{C_k\}_{k=1}^K$, the optimal mixture of spherical Gaussians is given by $\theta = \{(w_k,\mu_k,\sigma^2_k)\}_{k=1}^K$ with
\[ w_k = \frac{\abs{C_k}}{\abs{X}}\ ,\qquad \mu_k = \frac{1}{\abs{C_k}}\sum_{x_n\in C_k} x_n\ ,\qquad \sigma_k^2 = \frac{1}{d\abs{C_k}} \sum_{x_n\in C_k} \Vert x_n - \mu_k\Vert^2\ .\]
Fixing the Gaussian mixture model $\theta = \{(w_k,\mu_k,\sigma^2_k)\}_{k=1}^K$, the optimal partition $\cC=\{C_k\}_{k=1}^K$ is given by assigning each point to its most likely component, i.e. 
\[ x_n\in C_k \Leftrightarrow k=\argmax_{l\OneTo{K}} p(z_n=l|x_n,\theta)\ , \]
where 
\[ p(z_n=k|x_n,\theta) = \frac{w_k\cN(x_n|\mu_k,\sigma_k^2)}{\sum_{l=1}^K w_l\cN(x_n|\mu_l,\sigma_l^2)}\ ,  \]
which is the \emph{posterior probability} that $x_n$ has been generated by the $k$-th component of the given mixture.

If we repeatedly compute these update formulas, the solution converges to a local extremum or a saddlepoint of the likelihood function.

A proof of the correctenss of these update formulas (which we omit here) uses the following lemma.
\begin{lemma}\label{lem-wonderful}
Let $X\subset \IR^d$ be a finite set. Define 
\[ \mu(X)=\frac 1{\abs{X}} \sum_{x\in X} x\ .\]
Then, for all $y\in\IR^d$
\[ \sum_{x\in X}\norm{x-y}^2=\sum_{x\in X}\norm{x-\mu(X)}^2+\abs{X}\cdot \norm{y-\mu(X)}^2\ .\]
In particular, $\mu(X)=\argmin_{y\in \IR^d} \sum_{x\in X}\norm{x-y}^2$.
\end{lemma}

Note that an optimal CMLE solution is not changed by this algorithm.
Hence, an optimal CMLE solution is completely defined by a partition or a Gaussian mixture model.
Similarly, if we refer to a partition or a Gaussian mixture as a CMLE solution we assume that the missing parameters are as defined by the update formulas given above, respectively.

\subsection{Well-Defined Instances}

Unfortunately, the CMLE problem is not well defined in this form.
For example, you could choose $C_1 = \{x\}$ and $\mu_1 = x$ for some $x\in X$.
Then, as $\sigma_1 \rightarrow 0$ we get that $\cLc_K(X) \rightarrow -\infty$.
Consequently, we impose the following restrictions on instances.

\begin{definition}\label{def:well-defined}
We call $X = \dot\bigcup_{k=1}^K C_k$ a \emph{well-defined partition} if
\begin{enumerate}
	\item for all $k\OneTo{K}:\ \abs{C_k} \geq 2$.\label{rest:minpts2}
\end{enumerate}
We call $X$ itself a \emph{well-defined instance} if 
\begin{enumerate}
       \setcounter{enumi}{1}
	\item $\forall x,y\in X,x\neq y:\ \norm{x-y}^2 \geq \frac{4d}{\pi}$.\label{rest:dist}
\end{enumerate}
We denote $X = \dot\bigcup_{k=1}^K C_k$ as a \emph{well-defined solution} if $X$ is a well-defined instance and $\{C_k\}_{k=1}^K$ is a well-defined partition.
\end{definition}

In the following, we prove that, with these restrictions, the CMLE problem is well defined. 
That is, the minimum in Problem~\ref{prob-cmle} is well defined ($\cLc_K(X)> -\infty$).  
Moreover, we will see (Lemma~\ref{lem-lower-bound-variance}) that  for
the optimal solution we have $\sigma^2_k\ge \frac{1}{2\pi}$ or 
\begin{align}\label{eq1}
2\pi\sigma^2_k\ge 1 \quad \text{for $k\OneTo{K}$.}
\end{align}

First of all, note that the sum of squared distances between the points in $X$ and the mean $\mu(X)$ can be rewritten using pairwise distances (which are lower bounded in Restriction~\ref{rest:dist}).

\begin{lemma}\label{lem-1mean-characterization}
Let $X\subset \IR^d$ be a finite set and $\mu(X):=\frac{1}{\abs{X}} \sum_{x\in X}$ its mean, then
\begin{align*}
\sum_{x\in X} \norm{x-\mu(X)}^2=\frac{1}{2\abs{X}}\sum_{x\in X}\sum_{y\in X}\norm{x-y}^2.
\end{align*}
\end{lemma}
\begin{proof}
\begin{align*}
\sum_{x\in X}\sum_{y\in X}\norm{x-y}^2 & = \sum_{x\in X}\sum_{y\in X}\dpr{x-y}{x-y} & \\
 & = \sum_{x\in X}\sum_{y\in X}(\dpr{x}{x}+\dpr{y}{y}-2\dpr{x}{y} & \\
 & = 2\abs{X}\sum_{x\in X}\dpr{x}{x} -2 \sum_{x\in X}\sum_{y\in X}\dpr{x}{y} & \\
& =  2\abs{X}\sum_{x\in X}\dpr{x}{x}-2\abs{X}\sum_{x\in X}\dpr{x}{\mu(X)} & \\
& = 2\abs{X}\sum_{x\in X}\dpr{x}{x-\mu(X)}  & \\
& = 2\abs{X}\sum_{x\in X}\dpr{x-\mu(X)}{x-\mu(X)}  & \tag{using $\abs{X}\sum_{x\in X}\dpr{\mu(X)}{x-\mu(X)}=0$}\\
& = 2\abs{X}\sum_{x\in X}\norm{x-\mu(X)}^2.
\end{align*}
\end{proof}

Now using the restriction on the minimum pairwise difference between points (Restriction~\ref{rest:dist}) and on the minimum number of points (Restriction~\ref{rest:minpts2}) in a cluster, we can lower bound the variance of each cluster.
This directly yields Equation~\eqref{eq1} and our claim that the problem is well-defined under the restrictions given in Definition~\ref{def:well-defined}.

\begin{lemma}\label{lem-lower-bound-variance}
Let $Y$ be a subset of a set $X$ that satisfies Restriction~\ref{rest:dist} from Definition~\ref{def:well-defined} and that contains at least two different elements. Then, 
\[\sigma(Y)^2=\frac{1}{\abs{Y}d}\sum_{y\in Y}\norm{y-\mu(Y)}^2 \geq \frac{1}{2\pi}\ .\]
\end{lemma}
\begin{proof}
\begin{align*}
\sigma(Y)^2 & = \frac{1}{\abs{Y}d}\sum_{y\in Y}\norm{y-\mu(Y)}^2 \\
 & =   \frac{1}{2\abs{Y}^2d}\sum_{x\in Y}\sum_{y\in Y} \norm{x-y}^2 \tag{using Lemma~\ref{lem-1mean-characterization}} \\
 & \geq \frac{1}{2\abs{Y}^2d} \binom{\abs{Y}}{2} \min_{x,y\in Y,x\ne y} \norm{x-y}^2 &  \\
 & \geq \frac{1}{8d}  \min_{x,y\in Y,x\neq y} \norm{x-y}^2 \\
 & \geq \frac{1}{2\pi} \tag{using Restriction~\ref{rest:dist}}
\end{align*}
\end{proof}

Throughout the rest of this paper, we will restrict the search space of CMLE to well-defined solutions.
In particular, we only consider the optimal solution among all well-defined solutions.

\subsection{Well-Balanced Instances}

A central idea behind the algorithms that we present in this paper is that we do not allow somewhat \emph{degenerate} instances.
This means that we can find a function $f$ in the number of clusters that can be used to lower bound the number of points in a cluster and a function $g$ that can be used to lower bound the costs $\OPTc(C_k,1)$ of optimal clusters $C_k$.

\begin{definition}[well-balanced]\label{def:well-balanced}
Let $f,g:\IN \rightarrow \IR$.
We denote a partition $X = \dot\bigcup_{k=1}^K C_k$ as $f$-\emph{balanced} if for all $k\OneTo{K}$
\[ \abs{C_k} \geq \frac{\abs{X}}{f(K)}\ . \]

Furthermore, we denote the partition as an $(f,g)$\emph{-balanced} CMLE solution if it is $f$-balanced and additionally for all $k\OneTo{K}$
\[ \OPTc(C_k,1) \geq \frac{1}{g(K)} \cdot \sum_{k=1}^K \OPTc(C_k,1)\ . \]
\end{definition}

\begin{definition}
 Given a finite set $X\subset \IR^d$ and $K\in\IN$, we let
  \[ \OPTdiam(X,K) = \min_{\substack{\{C_1,\ldots,C_K\}, \\ \dot\cup_{k=1}^K C_k = X}} \ \max_{k\OneTo{K}}\  \max_{x,y\in C_k} \norm{x-y}\ . \]
\end{definition}

\begin{lemma}[From $f$-balanced to $(f,g)$-balanced]\label{lem:relation-between-balance-defs}
 An $f$-balanced solution $X = \dot\bigcup_{k=1}^K C_k$ is also an $\left(f,\Gamma\cdot f\right)$-balanced CMLE solution,
 where $\Gamma \leq 2\cdot\ln\left(32\pi \cdot \OPTdiam(X,K)\right) + \ln(K) + 1$. 
\end{lemma}
\begin{proof}
 \begin{align*}
    \OPTc(C_k,1) 
    &\geq \frac{\abs{C_k} d}{2} 
    \geq \frac{1}{f(K)} \frac{\abs{X} d}{2} \tag{due to Lemma~\ref{cor:lower-bound-nll} and $f$ balanced}\\
    &\geq \frac{1}{f(K)\cdot \Gamma} \cL_K(X) \tag{due to Lem.~\ref{lem:UpperBoundNLL}} \\
    &\geq \frac{1}{f(K)\cdot \Gamma} \sum_{k=1}^K \OPTc(C_k,1)\ .
 \end{align*}
\end{proof}

\section{Main Results (Theorems \ref{thm:cmlealgorithm} and \ref{thm:ABS})}

\begin{theorem}\label{thm:cmlealgorithm}
Let $X\subset \IR^d$, $K\in \IN$ and $\delta, \epsilon \in[0,1]$. If $X$ has an $(f,g)$-balanced optimal CMLE solution, then there exists an algorithm which computes  a mixture of $K$ spherical Gaussians  $\theta = \{(w_k,\mu_k,\sigma^2_k)\}_{k=1}^K$, such that
 \[ Pr \left[ \cLc_X(\theta) \leq (1+\epsilon)\OPTc(X,K)\right] \geq 1-\delta \ . \]
The runtime of the algorithm is bounded by
\begin{align*}
\abs{X}\cdot K \cdot \log(\Gamma)\cdot \log(g(K))\cdot 2^{\tilde\bigO\left( \frac{f(K)}{\epsilon\delta}\right)} 
\end{align*}
where $\Gamma \leq 2\cdot\ln\left(32\pi \cdot OPT_{diam}(X,K)\right) + \ln(K) + 1$.
\end{theorem}

\begin{corollary}
Let $X\subset \IR^d$, $K\in \IN$ and $\delta, \epsilon\in[0,1]$. 
If $X$ has an $f$-balanced optimal CMLE solution, then there exists an algorithm which computes  a mixture of $K$ spherical Gaussians $\theta$, such that
 \[ Pr \left[ \cLc_X(\theta) \leq (1+\epsilon)\OPTc(X,K)\right] \geq 1-\delta \ . \]
The runtime of the algorithm is bounded by
\begin{align*}
\abs{X}\cdot K \cdot \log(\Gamma)^2\cdot 2^{\tilde\bigO\left( \frac{f(K)}{\epsilon\delta}\right)} 
\end{align*}
where $\Gamma \leq 2\cdot\ln\left(32\pi \cdot OPT_{diam}(X,K)\right) + \ln(K) + 1$.
\end{corollary}

\begin{theorem} \label{thm:ABS} 
Let $X\subset \IR^d$, $K\in \IN$, and $\delta,\epsilon > 0$.
Let $\cC=\dot\bigcup_{k=1}^K C_k$ be a well-defined solution for the CMLE problem.
There is an algorithm that computes a mixture of $K$ spherical Gaussians $\theta$, 
such that
\[ \Pr\left[ \cLc_{X}(\theta) \leq  (1+\epsilon) \cLc_X(\cC) \right] \geq 1-\delta\ . \]
The running time of the algorithm is bounded by
\[ \abs{X}\,d\,\log\left(\frac{1}{\delta}\right)\,2^{\bigO\left( \frac{K}{\epsilon}\cdot \log\left(\frac{K}{\epsilon^2}\right) \right)}\, \left(\log(\log(\Delta^2))+1\right)^K \left( \log(f(K))\right)^K \ ,\]
where $\Delta^2 = \max_{x,y\in X} \{ \norm{x-y}^2\}$.
\end{theorem}

\section{Proof of Theorem~\ref{thm:cmlealgorithm}}

In the following we prove Theorem~\ref{thm:cmlealgorithm}.
\begin{itemize}
  \item In Section~\ref{sec:paramtocost} we show that, if the parameters of a CMLE solution are sufficently close to those of an optimal CMLE solution, then its complete-data log-likelihood is close to that of the optimal CMLE solution. 
  In Sections~\ref{sec:means} and \ref{sec:gridding} we then show how to obtain such parameter estimates.
  \item In Section~\ref{sec:means} we deal with the problem of estimating the means. 
  We use the superset sampling technique introduced by \cite{inaba94} to compute a set of candidate means which contains a good candidate, i.e. a  good estimation to the mean parameters of an optimal solution. 
  \item In Section~\ref{sec:gridding} we use a grid search to obtain estimates of the weights and variances.
  The core idea is to simply test all solutions lying on a specific grid in the search space. 
  By choosing a grid that is dense enough, we ensure that there are solutions on the grid which are sufficiently close to the parameters that we search for. 
\end{itemize}

\subsection{Estimate the Costs of Parameter Estimates}\label{sec:paramtocost}

For an optimal $(f,g)$-balanced CMLE solutions, we can estimate the parameters of the the respective optimal Gaussian mixture model and the likelihood of the optimal clusters.
We can show that the CMLE solution determined by these parameter estimates yields an approximation with respect to the complete data log-likelihood.

\begin{theorem}\label{thm:boundCMLEcost}
Let $X\subset \IR^d$, $K\in \IN$ and $\epsilon > 0$. 
Assume $X$ has an $f$-balanced optimal CMLE solution $X=\dot\bigcup_{k=1}^K C_k$ and let $(\tmu_1,\ldots,\tmu_K)$ such that for all $k\OneTo{K}$
\begin{align*}
	\norm{\tmu_k-\mu(C_k)}^2 \leq \frac{\epsilon}{\abs{C_k}}\sum_{x\in C_k}\norm{x-\mu(C_k)}^2\ .
\end{align*}
Let $(n_1,\dots,n_K)$, such that for all $k\OneTo{K}$
\begin{align}
	\abs{C_k} \leq n_k \leq (1+\epsilon)\abs{C_k}\ . \label{eq:boundCMLEcost:n_k}
\end{align}
and $\vec{\tsigma}=(\tsigma_1^2,\dots,\tsigma_K^2)\in\IR^K$, such that for all $k\OneTo{K}$ it holds 
\begin{align}
\tsigma_k^2 \geq \sigma_k^2 \label{eq:boundCMLEcost:tsigma-geq-sigma}
\end{align}
and
\begin{align}
 \ln(\tsigma_k^2) - \ln(\sigma_k^2)  \leq \left( (1+\epsilon)^2 - 1 \right) \frac{2}{\abs{C_k} d}\OPTc(C_k,1)\ .\label{eq:boundCMLEcost:tsigma-sigma-diff}
\end{align}
Define  $\tilde{\theta}=\{(\tw_k,\tmu_k,\tsigma_k^2)\}_{k=1,\ldots,K}$, where $\tw_k = \frac{n_k}{\sum_{l=1}^K n_l}$.
Then,
\[
 \cLc_X(\ttheta) 
 \leq (1+\epsilon)^{4} \OPTc(X,K).
\] 
\end{theorem}
\begin{proof}

Using that $\abs{C_l} \leq n_l \leq (1+\epsilon) \abs{C_l}$ for all $l=1,\ldots,K$, we obtain
$ \tw_k \geq \frac{1}{(1+\epsilon)}\cdot\frac{\abs{C_k}}{\abs{X}}$.
Hence, 
\begin{align*}
 -\ln(\tw_k) \cdot \abs{C_k} 
&\leq - \ln\left( \frac{1}{(1+\epsilon)}\cdot\frac{\abs{C_k}}{\abs{X}}\right)\abs{C_k} \tag{by Equation~\eqref{eq:boundCMLEcost:n_k}}\\
&\leq \ln(1+\epsilon)\abs{C_k} - \ln \left( \frac{\abs{C_k}}{\abs{X}} \right) \cdot \abs{C_k}\\
&\leq \eps \abs{C_k} - \ln \left( \frac{\abs{C_k}}{\abs{X}} \right) \cdot \abs{C_k} \tag{since $\ln(1+\epsilon)\leq \epsilon$}\\
&\leq \frac{2\epsilon}{d} \OPTc(C_k,1) - \ln \left( \frac{\abs{C_k}}{\abs{X}} \right) \cdot \abs{C_k} \tag{since $\OPTc(C_k,1) \geq \frac{\abs{C_k}\cdot d}{2}$}
\end{align*}

Furthermore, observe that
\begin{align*}
 \cL_{C_k}(\tmu_k,\tsigma_k) 
 &= \frac{\abs{C_k}d}{2}\ln(2\pi\tsigma_k^2) + \frac{1}{2\tsigma_k^2}\sum_{x\in C_k}\norm{x-\tmu_k}^2\\
 &\overset{\eqref{eq:boundCMLEcost:tsigma-geq-sigma}}{\leq} \frac{\abs{C_k}d}{2}\ln(2\pi\tsigma_k^2) + \frac{1}{2\sigma_k^2}\sum_{x\in C_k}\norm{x-\tmu_k}^2 \\
 &\leq \frac{\abs{C_k}d}{2}\ln(2\pi\tsigma_k^2) 
     + \frac{1}{2\sigma_k^2} (1+\epsilon) \sum_{x\in C_k}\norm{x-\mu_k}^2 \tag{By  Lemma~\ref{lem-wonderful} and property of $\tmu_k$}\\
 &= \frac{\abs{C_k}d}{2}\ln(2\pi\tsigma_k^2) 
     + (1+\epsilon)  \frac{\abs{C_k}d}{2} \tag{By  def. of $\mu_k$}\\
 &= \frac{\abs{C_k}d}{2} (\ln(2\pi)+\ln(\tsigma_k^2) )
     + (1+\epsilon)  \frac{\abs{C_k}d}{2} \\
 &\overset{\eqref{eq:boundCMLEcost:tsigma-sigma-diff}}{=} \frac{\abs{C_k}d}{2}
      \left(
	\ln(2\pi) 
	+ \left( (1+\epsilon)^{2} - 1 \right) \frac{2}{\abs{C_k} d}\OPTc(C_k,1)
	+\ln(\sigma_k^2) 
      \right)
     + (1+\epsilon)  \frac{\abs{C_k}d}{2} \\
 &= \frac{\abs{C_k}d}{2}\ln(2\pi\sigma_k^2) + (1+\epsilon)  \frac{\abs{C_k}d}{2} 
	+  \left( (1+\epsilon)^{2} - 1 \right) \OPTc(C_k,1)  \\
 &\leq (1+\epsilon)  \OPTc(C_k,1)
	+  \left( (1+\epsilon)^{2} - 1 \right) \OPTc(C_k,1)  \\
  &\leq  \left( (1+\epsilon)^{2} + \eps  \right) \OPTc(C_k,1)  \\
  &\leq (1+\epsilon)^{3} \OPTc(C_k,1)
\end{align*}

Overall, we have
\begin{align*}
  \cLc_X(\tilde{\theta}) 
  &= \sum_{k=1}^K  \cLc_{C_k}(\mu_k,\sigma^2_k) - \ln(w_k)\cdot \abs{C_k}  \\
  &\leq \sum_{k=1}^K  (1+\epsilon)^3 \OPTc(C_k,1) +  \frac{2\epsilon}{d} \OPTc(C_k,1) - \ln \left( \frac{\abs{C_k}}{\abs{X}} \right) \cdot \abs{C_k}  \\
   &= \sum_{k=1}^K  \left((1+\epsilon)^3+\frac{2\epsilon}{d}\right) \OPTc(C_k,1)  - \ln \left( \frac{\abs{C_k}}{\abs{X}} \right) \cdot \abs{C_k}  \\
   &\leq  \left((1+\epsilon)^3+\frac{2\epsilon}{d}\right) \sum_{k=1}^K \OPTc(C_k,1)  - \ln \left( \frac{\abs{C_k}}{\abs{X}} \right) \cdot \abs{C_k}\\
   &= \left((1+\epsilon)^3+\frac{2\epsilon}{d}\right)  \OPTc(X,K) \\
   &\leq (1+\epsilon)^4 \OPTc(X,K)
\end{align*}

\end{proof}

\subsection{Generate Candidate Means by Sampling}\label{sec:means}

We reuse the following well-known lemma on superset sampling.

\begin{lemma}[superset-sampling]\label{lem-superset-sampling}
Let $X\subset \IR^d$ be a finite set, $\alpha<1$ and $X'\subset X$ with $\abs{X'}\ge \alpha \abs{X}$. Let $S\subseteq X$ be a uniform sample multiset of size at least $\frac{2}{\alpha\epsilon\delta}$.
Then with probability at least $\frac{1-\delta}{5}$ there is a subset $S'\subseteq S$ with $\abs{S'}=\frac{1}{\epsilon\delta}$ such that
\begin{align*}
\norm{\mu(S')-\mu(X')}^2 \leq \frac{\epsilon}{\abs{X'}}\sum_{x\in X'}\norm{x-\mu(X')}^2.
\end{align*}
\end{lemma}

If we plug our notion of $f$-balanced solutions into this lemma, then we receive an algorithm that samples good approximative means. 

\begin{theorem}[sampling means]\label{thm:samplingalgo}
For a finite set $X\subset \IR^d$, $K\in \IN$ and $\epsilon, \delta > 0$, if $X = \dot\bigcup_{k=1}^K C_k$ is an $f$-balanced partition, then there is an algorithm that computes a set of 
$\log(1/\delta) \cdot 2^{\frac{K}{\epsilon\delta}\cdot \log\left(\frac{f(K)}{\epsilon\delta}\right)}$
$K$-tuples of points from $\IR^d$, such that with probability $1-\delta$ for one of these tuples it holds that for all $k\OneTo{K}$
\[\norm{\mu_k-\mu(C_k)}^2 \leq \frac{\epsilon}{\abs{C_k}}\sum_{x\in C_k}\norm{x-\mu(C_k)}^2\ .\]
The runtime of the algorithm is bounded by 
$\log(1/\delta)\cdot K \cdot \left(\abs{X} + 2^{\frac{K}{\epsilon\delta}\cdot \log\left(\frac{f(K)}{\epsilon\delta}\right)}\right)$.
\end{theorem}

\begin{proof}
	Consider the following algorithm, which computes a candidate set of tuples of means.
	\begin{algorithm}[H]
	\KwIn{
		$X\subset \IR^d$ : input points \\
		$K\in\IN$ : number of clusters
	}
	\KwOut{
		set of candidate tuples of means
	}
	    $P \leftarrow \emptyset$\;
    	\For{$k=1,\ldots,K$}
    	{
			sample a multiset $S$ of size $\frac{1}{\alpha \epsilon \delta}$ from $X$\;
			$T \leftarrow \left\{ \mu(S') | S'\subset S, \abs{S'} = \lceil \frac{1}{\epsilon\delta}\rceil \right\}$\;
			$P \leftarrow P\times T$\;
	    }
    	\Return $P$\;
	\caption{\textsc{Approx-Means}$(X,K)$}\label{alg-ABSa}
	\end{algorithm}
	Using Lemma~\ref{lem-superset-sampling} with $\alpha = \frac{1}{f(K)}$, we know that the output of a single run of \textsc{Approx-Means} contains a tuple with the desired property with probability $\left(\frac{1-\delta}{5}\right)^K$. 
	
	We know that 
	\[\abs{T} \leq \left(\frac{1}{\alpha\epsilon\delta}\right)^{\frac{1}{\epsilon\delta}},\]
	thus
	\[\abs{P} = \abs{T}^K \leq 2^{\frac{K}{\epsilon\delta}\cdot \log\left(\frac{f(K)}{\epsilon\delta}\right)}.\]
	The runtime is bounded by
	\[ K\cdot \abs{X} + \sum_{k=1}^K \abs{T}^k \leq K \left(\abs{X} + 2^{\frac{K}{\epsilon\delta}\cdot \log\left(\frac{f(K)}{\epsilon\delta}\right)}\right).\]
	
	By executing \textsc{Approx-Means} $\log(1/\delta)$ times we receive the desired success probability.

\end{proof}

\subsection{Generate Candidate Cluster Sizes and Variances by Using Grids}\label{sec:gridding}

So far, we have formulated an algorithm that gives us good means.
In the following, we will use the gridding technique to determine a set of candidates for the the cluster sizes and variances.
First of all, we generate a set of cluster sizes that contains good approximations of the cluster sizes of any $f$-balanced solutions.
Then, we approximate the negative log-likelihood of optimal CMLE clusters, i.e. $\sum_{k=1}^K \OPTc(C_k,1)$ where the $C_k$ are the optimal CMLE clusters.
Then, we present how to construct a candidate set of variances that contains good estimates of the variances of any $(f,g)$-balanced optimal CMLE solution.

\subsubsection{Grid Search for Cluster Sizes}\label{subsec:clustersizes}

\begin{theorem}\label{thm:clustersizes}
	Let $X\subset \IR^d$, $K\in \IN$ and let $X=\dot\bigcup_{k=1}^K C_k$ be an $f$-balanced partition. 
	Then there exists an algorithm that outputs a set 
	$S\subseteq \IN^K$, 
	$\abs{S} = \left( \frac{\log(f(K))}{\log(1+\eps)} \right)^K $,
	that contains a tuple $(n_1, \dots, n_K)\in S$ such that 
	\begin{align}
		\abs{C_k} \leq n_k \leq (1+\epsilon) \abs{C_k}.
	\end{align}
	for all $k\OneTo{K}$.
\end{theorem}

\begin{proof}
Since we assume a $f$-balanced solution, we know that for all $k\OneTo{K}$
\[\frac{\abs{X}}{f(K)} \leq \abs{C_k} \leq \abs{X}.\]
Thus, there exist a value $i^* \in \{ 1, \dots, \lceil\log_{1+\epsilon}(f(K))\rceil\}$ such that
\[(1+\epsilon)^{i^*-1}\frac{\abs{X}}{f(K)} \leq \abs{C_k} \leq (1+\epsilon)^{i^*}\frac{\abs{X}}{f(K)}.\]
Thus, we receive $\lceil\log_{1+\epsilon}(f(K))\rceil$ many values for each cluster size $n_k$.
The algorithm outputs all possible combinations of these values.
\end{proof}

\subsubsection{Bounds on the Log-Likelihood of optimal CMLE clusters}\label{subsec:nll}

Lemma~\ref{lem-lower-bound-variance} provides us with a lower bound on the negative log-likelihood of a cluster.

\begin{corollary}[Lower Bound on the Optimal Log-Likelihood]\label{cor:lower-bound-nll}
 Let $X=\dot\bigcup_{k=1}^K C_k$ be an optimal CMLE solution.
 Then, $\OPTc(C_k,1) \geq \frac{\abs{C_k}d}{2}$.
\end{corollary}

The next step is to find an upper bound on the optimal complete-data likelihood value.
We use Gonzales algorithm to compute a value that gives us a tighter bound than just the maximum spread (over the dimensions of the vectors in the data set).

\begin{lemma}[Upper Bound on the Optimal Complete-Data Log-Likelihood]\label{lem:UpperBoundNLL}
Let $X\subset\IR^d$ and $K\in \IN$. 
A Value $\Gamma$ can be computed in time $\bigO(K\cdot d\cdot\abs{X})$ such that the complete-data likelihood of an optimal CMLE solution can be bounded by 
\[\OPTc(X,K)\leq \frac{\abs{X}d}{2}\cdot\Gamma\] 
and $\Gamma = \ln(2\pi s^2) + 1 + \ln(K)$ for some $s\leq 4\cdot OPT_{diam}(X)$.
\end{lemma}

\begin{proof}
  Run Gonzales algorithm. 
  The output is a set of $K$ points $p_1,\ldots,p_K \in X$. 
  Compute the point $z$ with maximum distance to its closest point in $\{p_1,\ldots,p_K\}$ and set $s := \min_{k=1,\ldots,K} \norm{z-p_k}$.
  Consider the solution where the $p_k$ are the centers. 
  Partition the points into point sets $\cC = \{C_1,\ldots,C_K\}$, with $\norm{x-p_k} = \min_{i=1,\ldots,K} \norm{x-p_i}$ for all $x \in C_k$.
  Notice that the distances between any point and its center is at most $s$. 
  Thus, when computing the optimal variance in each cluster, it is at most $s^2$. 
  Then, for $\theta=\left\{\left(\frac{1}{K},p_k,\sigma(X_k,p_k)\right)\right\}_{k=1}^K$ we have
  \begin{align*}
  \OPTc(X,K) \le \cLc_X(\theta,\cC) 
  &= \sum_{k=1}^K \frac{\abs{C_k}d}{2} \ln(2\pi \sigma(C_k,p_k)^2)+ \frac{\abs{C_k} d}{2 } - \ln(w_k)\cdot \abs{C_k} \\
  &\le \left(\sum_{k=1}^K \frac{\abs{C_k}d}{2} \ln(2\pi s^2)+ \frac{\abs{C_k} d}{2 } \right)- \ln\left(\frac{1}{K}\right)\cdot \abs{X} \\
  &=  \frac{\abs{X}d}{2} \ln(2\pi s^2) + \frac{\abs{X} d}{2 }  +  \ln(K)\cdot \abs{X}\\
  &\leq \frac{\abs{X}d}{2} \left( \ln(2\pi s^2) + 1  +  \ln(K)\right)
  \end{align*}
\end{proof}

Given two bounds, we can find a constant factor approximation of the the sum of the negative log-likelihoods of optimal CMLE clusters, i.e. $\sum_{k=1}^K \OPTc(C_k,1)$,  using a grid search.

\begin{lemma}[Estimating the Optimal Log-Likelihood]\label{lem:OPTestEst}
Let $X\subset \IR^d$, $K\in \IN$, and $\epsilon > 0$. 
Let $X=\dot\cup_{k=1}^K C_k$ be an optimal CMLE solution.
Then, there exists a set of 
$\log(3\Gamma/d)/\log(1+\epsilon)$
many values which contains a value $\cN_{est}$ with
\[ \frac{1}{1+\epsilon} \cN_{est} \leq \sum_{k=1}^K \OPTc(C_k,1) \leq \cN_{est}\ .\]
\end{lemma}
\begin{proof}
Combining Corollary~\ref{cor:lower-bound-nll} and Lemma~\ref{lem:UpperBoundNLL}, we know that 
\[\frac{\abs{X}d}{2} \leq \sum_{k=1}^K \OPTc(C_k,1) \leq \OPTc(X,K) \leq \frac{\abs{X}d}{2} \Gamma . \]
Thus, there exist a value $i^* \in \{ 1, \dots, \lceil \log_{1+\epsilon}(\Gamma)\rceil\}$ such that
\[(1+\epsilon)^{i^*-1}\frac{\abs{X}d}{2} \leq \sum_{k=1}^K \OPTc(C_k,1)\leq (1+\epsilon)^{i^*}\frac{\abs{X}d}{2}.\]
The algorithm outputs all $\lceil \log_{1+\epsilon}(\Gamma)\rceil$ values.
\end{proof}

Given this approximation of the sum of the negative log-likelihoods, we will be able to find an approximation of the negative log-likelihoods of a single cluster as we will see in the next section.

\subsubsection{Grid Search for Variances}\label{subsec:variancegridding}

Given the approximations of the size of the clusters and their negative log-likelihod, we are now able to find estimates of the variances.

\begin{theorem}\label{thm:variancegridding}
Let $X\subset \IR^d$, $K\in \IN$ and $\epsilon > 0$. Assume $X$ has an $(f,g)$-balanced CMLE solution $X=\dot\bigcup_{k=1}^K C_k$.
Let additionally $\cN_{est}\in \IR$, with
\begin{align}
	\frac{1}{1+\epsilon}\cN_{est}\leq \sum_{k=1}^K \OPTc(C_k,1) \leq \cN_{est},\label{prop:estimatedNLLValue}
\end{align}
and $(n_1,\dots,n_K)$, such that for all $k\OneTo{K}$
\begin{align}
	\abs{C_k} \leq n_k \leq (1+\epsilon)\abs{C_k}.\label{prop:estimatedSizes}
\end{align}
Then there exists an algorithm that computes a set of size $K \cdot \frac{\log(g(K))}{\log(1+\epsilon)}$, that contains a tuple $(\tsigma_1^2,\dots,\tsigma_K^2)$, such that for all $k\OneTo{K}$ it holds 
\begin{align}
\tsigma_k^2 \geq \sigma_k^2 
\end{align}
and
\begin{align}
 \ln(\tsigma_k^2) - \ln(\sigma_k^2)  \leq \left( (1+\epsilon)^2 - 1 \right) \frac{2}{\abs{C_k} d}\OPTc(C_k,1)\ .
\end{align}
\end{theorem}
\begin{proof}
Observe that
\begin{align*}
\frac{1}{g(K)(1+\epsilon)}\cN_{est}
\leq \frac{1}{g(K)}\sum_{k=1}^K \OPTc(C_k,1)
\overset{\text{Def.~}\ref{def:well-balanced}}{\leq} \OPTc(C_k,1)
\leq  \sum_{k=1}^K \OPTc(C_k,1) 
\leq \cN_{est}.
\end{align*}

Thus, there exists a value $j^* \in \left\{ \lceil -\log_{1+\epsilon}(g(K))\rceil, \dots, 0\right\}$ which satisfies
\begin{align*}
(1+\epsilon)^{j^*-1}\cN_{est} \leq \OPTc(C_k,1) \leq (1+\epsilon)^{j^*}\cN_{est}\ .
\end{align*}
Denote the upper bound by $\hat{\cN} \coloneqq (1+\epsilon)^{j^*} \cN_{est}$ and set $\tsigma^2_k \coloneqq \exp\left( \frac{2(1+\epsilon)}{{n}_k d}\hat{\cN}-\ln(2\pi)-1 \right)$.

Notice that
\begin{align*}
 \OPTc(C_k,1) = \cL_{C_k}(\mu_k,\sigma_k^2) = \frac{\abs{C_k}d}{2}\left( \ln(2\pi\sigma_k^2 + 1)\right) \\
 \Leftrightarrow \ln(\sigma_k^2) =  \frac{2}{\abs{C_k}d}\OPTc(C_k,1)-\ln(2\pi)-1
\end{align*}
Thus, 
\begin{align*}
 \ln(\tsigma_k^2 ) = \frac{2(1+\epsilon)}{{n}_k d}\hat{\cN}-\ln(2\pi)-1
 \geq \frac{2}{\abs{C_k} d}\OPTc(C_k,1)-\ln(2\pi)-1 = \ln(\sigma_k^2)
\end{align*}
and
\begin{align*}
 \ln(\tsigma_k^2) - \ln(\sigma_k^2) 
 & = \frac{2(1+\epsilon)}{{n}_k d}\hat{\cN} - \frac{2}{\abs{C_k} d}\OPTc(C_k,1)\\
 &\leq \frac{2(1+\epsilon)^2}{ \abs{C_k} d}\OPTc(C_k,1) - \frac{2}{\abs{C_k} d}\OPTc(C_k,1)\\
 &= \left( (1+\epsilon)^2 - 1 \right) \frac{2}{\abs{C_k} d}\OPTc(C_k,1)\label{Eq:DiffSigmas}
\end{align*}
\end{proof}

\section{Proof of Theorem~\ref{thm:ABS}}

In the following we present the proof of Theorem~\ref{thm:ABS}. 

\begin{itemize}
  \item In Section~\ref{sec:proof2:gridding} we show how to estimate the variances and the cluster sizes of a well-defined CMLE solution via gridding. 
   The idea behind a grid search is simply to test all solutions lying on a grid in the search space. 
  By choosing a grid that is dense enough, we ensure that there are solutions on the grid which are sufficiently close to the parameters that we search for.
  \item In Section~\ref{sec:proof2:abs}, we show how one can find good estimates of the means when given good estimates of the weights and covariances. 
  To this end, we adapt the sample-and-prune technique presented in \cite{ABS}. 
\end{itemize}

\subsection{Generate Candidates for Variances and Weights}\label{sec:proof2:gridding}

\begin{lemma}\label{lem:ucmle:gridding}
	Let $X \subset \IR^d$, and $\{C_k\}_{k=1}^K$ be a well-defined CMLE solution for $X$, with corresponding variances $\{\sigma_k^2\}_{k=1}^K$.
	Then, there exists an algorithm which outputs a set of at most $\left(\frac{\log(\log(\Delta^2))+1}{\log(1+\epsilon)}\right)^K$ tuples of variances, which contains a tuple $(\tilde\sigma_k^2)_{k=1}^K$, such that
	\[ \forall k\OneTo{K}: \sigma_k^2 \leq \tilde\sigma_k^2 \leq (\sigma_k^2)^{(1+\epsilon)}\ ,  \]
	where $\Delta^2 = \max_{x,y\in X} \{ \norm{x-y}^2\}$.
\end{lemma}

\begin{proof}

	We know that optimal variances $\sigma_k^2$ of a well-defined solution are bounded from below by
	\[ \forall k\OneTo{K}: \frac{1}{2\pi} \leq \sigma_k^2 . \]
	Furthermore, we know that these are also bounded from above by
	\begin{align*} 
		\forall k\OneTo{K}: \sigma_k^2 
		= \frac{1}{\abs{C_k}d}\sum_{x\in C_k} \norm{x-\mu(C_k)}^2 
		\leq \frac{1}{\abs{C_k}d}\sum_{x\in C_k} \Delta^2 \leq \Delta^2\ .
	\end{align*}

	Because $1/(2\pi) \leq \sigma_k^2 \leq \Delta^2$, there exists a value 
	\[ k^* \in \{1,\dots, \log_{1+\epsilon}(-\log_{1/(2\pi)}(\Delta^2))\}\] 
	such that
	\[ \left(1/(2\pi)\right)^{(1+\epsilon)^{k^*-1}} \leq \sigma_i^2 \leq \left(1/(2\pi)\right)^{(1+\epsilon)^{k^*}} . \]
	Thus, we receive $\left\lceil\frac{\log(\log(\Delta^2))-\log(\log(2\pi))}{\log(1+\epsilon)}\right\rceil$ many values for each variance.
	The algorithm outputs all possible combinations of these values.
\end{proof}

The following result is the same as in Section~\ref{sec:gridding}. 

\begin{theorem}
	Let $X\subset \IR^d$, $K\in \IN$ and let $\cC=\dot\bigcup_{k=1}^K C_k$ be an $f$-balanced partition. 
	Then there exists an algorithm that outputs a set 
	$S\subseteq \IN^K$, 
	$\abs{S} = \left( \frac{\log(f(K))}{\log(1+\eps)} \right)^K $,
	that contains $\{n_1, \dots, n_K\}\subset S$ such that 
	\begin{align}
		\abs{C_k} \leq n_k \leq (1+\epsilon) \abs{C_k}.
	\end{align}
	for all $k\OneTo{K}$.
\end{theorem}

\subsection{Applying the ABS Algorithm}\label{sec:proof2:abs}

\begin{algorithm}
\KwIn{ \\
$R\subset X \subset \IR^d\ $: set of remaining input points\\
$l\in\IN\ :$ number of means yet to be found\\
$\vec\mu = (\mu_1,\ldots,\mu_j)\ :$ tuple of $j\leq k-l$ candidate means\\
$(\tsigma_{1}^2,\ldots,\tsigma_{k}^2)\ :$ vector of $k$ variances\\
$(\tw_{1}^2,\ldots,\tw_{k}^2)\ :$ vector of $k$ weights\\
\\
\textbf{Notation:}\\
$\vec S\ :$ vector containing the elements of set $S$ in arbitrary order \\
$\vec x \circ \vec y:\ $ concatenation of vectors, i.e. for $\vec x = (x_1,\ldots,x_n)$ and $\vec y =(y_1,\ldots,y_m)$,\newline
\phantom{$\vec x \circ \vec y:\ \ $}$\vec x \circ \vec y = (x_1,\ldots,x_n,y_1,\ldots,y_m)$
}
\KwOut{ $\theta = \{ (w_i,\mu_i,\sigma_i) \}$ containing at most $k$ tuples of mean and variance 
}
  \eIf{$l=0$}{\Return $\vec P$\;}
  {
    \eIf{$l\geq \abs{R}$}
    {
      \Return $\theta = \{ (\mu_i,\sigma_i) \}_i$ where $\vec\mu \circ \vec R = (\mu_i)_i$\;
    }
    {
	\emph{/* sampling phase */}\;
	sample a multiset $S$ of size $\frac{1}{\alpha \epsilon \delta}$ from $R$\;
	$T \leftarrow \left\{ \mu(S') | S'\subset S, \abs{S'} = \frac{1}{\epsilon\delta} \right\}$\;
	$\cM_k \leftarrow \emptyset$\;
	\For{$t\in T$}
	{
		$\cM_k \leftarrow \cM_k\cup \textsc{Approx-Means}(R, l-1, \{ \vec\mu \circ (t) | \vec\mu \in \cM_{k-l}\}, \Sigma)$\;
	}
	\emph{/* pruning phase */}\;
	    $N \gets$ set of $\frac{\abs{R}}{2}$ points $x$ from $R$ with smallest minimum negative complete-data log-likelihood cost wrt. the weighted component given by
	    $(\tw_i, \mu_i,\tsigma_i^2)$ for $i\OneTo{j}$, i.e.
	    \[ \min_{i\OneTo{j}} \left\{ \frac{d}{2}\ln(2\pi \tsigma_i^2) + \frac{1}{2\tsigma_i^2} \norm{x-\mu_i}^2 - \ln(\tw_i) \right\} \]
	    $\cM_k \gets \cM_k \cup \textsc{Approx-Means}(R\setminus N, l, \cM_{k-l}, \Sigma)$\;
	\Return the candidate $\theta = \{ (w_i, \mu_i,\sigma_i) \}_i$, $(\mu_i)\in\cM_k$, which has minimal cost $\cLc_X(\theta)$ \;
    }
  }
\caption{\textsc{Approx-Means}$(R, l,\cM_{k-l},\Sigma)$}\label{alg-ABS}
\end{algorithm}

In the following we analyze Algorithm~\ref{alg-ABS}.
We show that the algorithm can be used to construct means such that, together with appropriate approximations of the weights and variances, we obtain a CMLE solution with costs close to the costs of the given CMLE solution.

\begin{theorem}\label{thm:ucmle:abs}
 Let $\tsigma_i \in [\sigma_i^2, (\sigma_i^2)^{(1+\epsilon)}]$
 and $\tw_k \geq \frac{1}{(1+\epsilon)}w_k$ for $i\OneTo{k}$.
 Algorithm~\ref{alg-ABS} started with $(X,k,\emptyset,(\tsigma_1^2,\ldots,\tsigma_k^2))$ computes a tuple  $(\tmu_1,\ldots,\tmu_k)$ such that with probability at least $\left(\frac{1-\delta}{5}\right)^k$ 
 \[ \cLc_{X}((\tw_i, \tmu_i,\tsigma_i^2)_{i\OneTo{k}}) \leq (1+\epsilon) \cLc(X)  \ .\]

 The running time of the algorithm is bounded by
$\abs{X}\,d\,2^{\bigO(k/\epsilon\cdot \log(k/\epsilon^2))}$.
\end{theorem}

Let $\dot{\bigcup}_{i=1}^k C_i$ be a partition of $X$ into optimal CMLE clusters.
We introduce 
\[ C_{[i,j]} = \dot{\bigcup}_{t=i}^j C_t \] 
as a short notation for the disjoint union of clusters $i$ through $j$.
We assume that the $C_i$ are numbered by the order their approximate means $\tilde{\mu}_i$ are found by the superset-sampling technique.

Now, let $X=R_0 \supseteq R_i \supseteq \dots \supseteq R_{k-1}$ be a sequence of input sets computed by the algorithm, such that 
\[ \abs{C_i \cap R_{i-1}} \geq \alpha \abs{R_{i-1}}. \]
Without loss of generality assume that each $R_i$ is the largest of these sets with this property. 

By using Lemma~\ref{lem-superset-sampling}, we obtain the following Lemma.
\begin{lemma}[By Superset-Sampling]
  With probability at least $((1-\delta)/5)^k$ we have
\[ \norm{\tilde{\mu}_i - \mu(C_i\cap R_{i-1})}^2 \leq \frac{\epsilon}{\abs{C_i \cap R_{i-1}}}\sum_{x \in C_i\cap R_{i-1}} \norm{x - \mu(C_i\cap R_{i-1})}^2 \]
for all $i\OneTo{K}$.
\end{lemma}

By $N_i \coloneqq R_{i-1} \setminus R_i$ we denote the set of points remove between two sampling phases. Using these definitions we can see that 
\[ \dot{\bigcup}_{i=1}^k \left( C_i \cap R_{i-1}\right) \; \dot{\cup} \; \dot{\bigcup}_{i=1}^k \left( C_{[i+1,k]} \cap N_i \right) \]
is a disjoint partition of $X$. 
Each set $C_i \cap R_{i-1}$ on the left side contains the points that the mean $\tilde{\mu}_i$ has been sampled from. 
The sets $C_{[i+1,k]} \cap N_i$ on the right side contain points incorrectly assigned to $\{\tilde{\mu}_1, \dots, \tilde{\mu}_i\}$ during the pruning phases between the sampling of $\tilde{\mu}_i$ and $\tilde{\mu}_{i+1}$.

Denote by $\theta_i$ the parameters of the first $i$ weighted Gaussians obtained by the algorithm, i.e. 
\[ \ttheta_i  = ((\tw_1, \tmu_1,\tsigma_1),\ldots,(\tw_i, \tmu_i,\tsigma_i))\ .\]

\begin{lemma}[cf. Claim~4.8 in \cite{Ackermann09}]\label{lem:ucmle:wrongly-assinged}
 \[ \cLc_{C_{[i+1,k]}\cap N_i} (\ttheta_i) \leq 8\alpha k \cLc_{C_{[1,i]}\cap R_{i-1}} (\ttheta_i)  \]
\end{lemma}
\begin{proof}
 As in \cite[p.~70ff]{Ackermann09}, with ``$\cost$`` replaced by ''$\cLc$``.
\end{proof}

Denote by $\cost(P,C)$ the $k$-means cost of a point set $P$ wrt. a set of means $C$.

\begin{lemma}[cf. Claim~4.9 in \cite{Ackermann09}]\label{lem:ucmle:kmeans-costs}
 For every $i\OneTo{k}$ we have
 \[ \cost(C_i\cap R_{i-1},\tmu_i) \leq (1+\epsilon) \cost(C_i,\mu_i)\ .\]
\end{lemma}
\begin{proof}
 As in \cite[p.~70ff]{Ackermann09}, using that optimal means in CMLE are means of the optimal CMLE clusters.
\end{proof}

Given appropriate approximate variances, we can conclude that a similar bound holds wrt. the complete-data log-likelihood.

\begin{lemma} \label{lem:ucmle:correctly-assigned}
  Given $\tsigma_i \in [\sigma_i^2, (\sigma_i^2)^{(1+\epsilon)}]$ and 
  $\tw_i = \frac{n_i}{\abs{X}}$ with $n_i\in[\abs{C_i},(1+\epsilon)\abs{C_i}]$, we have
 \[ \cLc_{C_i\cap R_{i-1}}(\tw_i, \tmu_i,\tsigma_i^2) \leq (1+\epsilon) \cLc_{C_i}(w_i, \mu_i,\sigma_i^2)\ .\]
\end{lemma}
\begin{proof}
 \begin{align*}
  \cLc_{C_i\cap R_{i-1}}(\tmu_i,\tsigma_i^2)
  &= \frac{ \abs{C_i\cap R_{i-1}}d }{2}  \ln(2\pi\tsigma_i^2) 
  + \frac{1}{2\tsigma_i^2} \cost(C_i\cap R_{i-1},\tmu_i)
  - \abs{C_i\cap R_{i-1}}\ln(\tw_i)\ .
 \end{align*}
 We have
 \begin{align*}
  \ln(2\pi\tsigma_i^2) \leq \ln(2\pi(\sigma_i^2)^{(1+\epsilon)})= (1+\epsilon) \ln(2\pi\sigma_i^2) \ .
 \end{align*}
 Furthermore,
 Using that $\abs{C_l} \leq n_l \leq (1+\epsilon) \abs{C_l}$ for all $l=1,\ldots,K$, we obtain
$ \tw_k \geq \frac{\abs{C_k}}{\abs{X}}$.
Hence, 
\begin{align*}
  -\ln(\tw_i) \cdot \abs{C_i \cap R_{i-1}}  &\leq -\ln(\tw_i) \cdot \abs{C_i} \\
&\leq - \ln\left( \frac{\abs{C_i}}{\abs{X}}\right)\abs{C_i} \tag{by Equation~\eqref{eq:boundCMLEcost:n_k}}\\
&= - \ln \left(w_i\right) \cdot \abs{C_i} 
\end{align*}

 By Lemma~\ref{lem:ucmle:kmeans-costs} and $\tsigma_i^2 \geq \sigma_i^2$,
 \begin{align*}
  \frac{1}{2\tsigma_i^2} \cost(C_i\cap R_{i-1},\tmu_i)
  &\leq (1+\epsilon) \frac{1}{2\sigma_i^2} \cost(C_i,\mu_i)\ .
 \end{align*}
 From this and by using that $\sigma_i^2 =\frac{1}{\abs{C_i}d}\cost(C_i,\mu_i)$, we conclude
  \begin{align*}
  \cLc_{C_i\cap R_{i-1}}(\tmu_i,\tsigma_i^2)
  &\leq 
  (1+\epsilon) \frac{ \abs{C_i}d }{2}   \ln(2\pi\sigma_i^2)  
  +(1+\epsilon) \frac{1}{2\sigma_i^2} \cost(C_i,\mu_i)
   - \ln(w_i)\abs{C_i}\\
  &\leq 
  (1+2\epsilon)  \cN_1(C_i) - \ln(w_i)\abs{C_i}\\
  &\leq (1+ 2\epsilon) \cLc_{C_i}(\mu_i,\sigma_i^2)\ .
 \end{align*}

  Running Algorithm~\ref{alg-ABS} with $\epsilon/3$ instead of $\epsilon$ yields the claim. 
\end{proof}

Analogously to \cite{Ackermann09}, we can prove Theorem~\ref{thm:ucmle:abs} as follows.

\begin{proof}[Proof of Theorem~\ref{thm:ucmle:abs}] 
  Let $\ttheta_k = (\tmu_i,\tsigma_i^2)_{i\OneTo{k}}$. Then,
  \begin{align*}
   \cLc_{X}(\ttheta_k)
   & \leq \sum_{i=1}^k  \cLc_{C_i\cap R_{i-1}}(\tmu_i,\tsigma_i^2) 
          + \sum_{i=1}^{k-1} \cLc_{C_{[i+1,k]}\cap N_i} (\ttheta_k)\\
   & \leq \sum_{i=1}^k  \cLc_{C_i\cap R_{i-1}}(\tmu_i,\tsigma_i^2) 
          + 8\alpha k \sum_{i=1}^{k-1}  \cLc_{C_{[1,i]}\cap R_{i-1}} (\ttheta_k)
          \tag{due to Lemma~\ref{lem:ucmle:wrongly-assinged}}\\
   & \leq \sum_{i=1}^k  \cLc_{C_i\cap R_{i-1}}(\tmu_i,\tsigma_i^2) 
          + 8\alpha k \sum_{i=1}^{k-1} \sum_{t=1}^i  \cLc_{C_t\cap R_{i-1}} (\tmu_t,\tsigma_t^2)\ .
  \end{align*}
  Since $R_i\subseteq R_{i-1}$, we have $C_t\cap R_{i-1}\subseteq C_t \cap R_{t-1}$.
  Hence,
  \begin{align*}
   \sum_{i=1}^{k-1} \sum_{t=1}^i  \cLc_{C_t\cap R_{i-1}} (\tmu_t,\tsigma_t^2) 
   \leq \sum_{i=1}^{k-1} \sum_{t=1}^i  \cLc_{C_t\cap R_{t-1}} (\tmu_t,\tsigma_t^2) \\
    \leq k \sum_{i=1}^{k-1}  \cLc_{C_i\cap R_{i-1}} (\tmu_i,\tsigma_i^2) \ .
  \end{align*}
  Thus,
  \begin{align*}
   \cLc_{X}(\ttheta_k) 
   & \leq \sum_{i=1}^k  \cLc_{C_i\cap R_{i-1}}(\tmu_i,\tsigma_i^2) 
          + 8\alpha k^2 \sum_{i=1}^{k-1}  \cLc_{C_i\cap R_{i-1}} (\tmu_i,\tsigma_i^2) \\
   & \leq (1+8\alpha k^2) \sum_{i=1}^k  \cLc_{C_i\cap R_{i-1}}(\tmu_i,\tsigma_i^2) \\
   & \leq (1+8\alpha k^2)(1+\epsilon) \cLc(X) \tag{by Lemma~\ref{lem:ucmle:correctly-assigned}}\ .
  \end{align*}

 Finally, running the algorithm for $\epsilon := \epsilon/2$ and $\alpha=\theta(\epsilon/k^2)$  yields the theorem.
  
\end{proof}

\section{Special Cases}

\subsection{Weighted $K$-Means (Identical Covariances)}

In this section we consider a restricted version of the CMLE problem where we are only interested in Gaussian mixture models where all components share the same fixed spherical covariance matrix, i.e. parameters $\theta=\{(w_k,\mu_k,\Sigma_k)\}_{k\OneTo{K}}$ where $\Sigma_k= \frac{1}{2\beta}I_d$ for all $k\OneTo{K}$.
We call this problem the \emph{Weighted $K$-Means} (WKM) problem.

\begin{problem}[WKM]\label{prob-wkm}
Given a finite set $X\subset \IR^d$ and an integer $K\in \IN$, find a partition $\cC = \{C_1,\ldots,C_K\}$ of $X$ into $K$ disjoint subsets and $K$ weighted means $\theta = \{(w_k, \mu_k)\}_{k=1}^K$, where $\mu_k\in \IR^D$, $w_k\in\IR$, and  $\sum_{k=1}^K w_k=1$, minimizing
\begin{align*}
\cL^{wm}_X(\theta, \cC)
&=\sum_{k=1}^K  \beta \left( \sum_{x\in C_k} \norm{x-\mu_k}^2\right) -  \ln(w_k)\cdot\abs{C_k}   \ .
\end{align*}
We denote the minimal value by $OPT_{wm}(X,K)$.
\end{problem}

\begin{corollary}
Let $X\subset \IR^d$, $K\in \IN$, and $\delta,\epsilon > 0$.
Let $X=\dot\bigcup_{k=1}^K C_k$ be a well-defined solution for the WKM problem.
There is an algorithm that computes $K$ weighted means
$\theta = \{(\tw_k, \tmu_k)\}_{k=1}^K$
such that with probability at least $1-\delta$
\[ \cL^{wm}_{X}((\tw_i,\tmu_i)_{i\OneTo{K}}) \leq  (1+\epsilon) OPT_{wm}(X,K)\ . \]
The running time of the algorithm is bounded by
\[ \abs{X}\,d\,2^{\bigO(K/\epsilon\cdot \log(K/\epsilon^2))} \cdot \left( \log(f(K))\right)^K \ .\]
\end{corollary}
\begin{proof}
 Use a grid search to obtain candidates for the weights, then apply the ABS algorithm.
\end{proof}

\subsection{Uniform Weights}

In this section we consider a restricted version of the CMLE problem where we are only interested in Gaussian mixture models with fixed uniform weights, i.e. parameters $\theta=\{(w_k,\mu_k,\Sigma_k)\}_{k\OneTo{K}}$ where  $w_k=1/K$ for all $k\OneTo{K}$.
We denote this problem by \emph{Uniform Complete-Data Maximum Likelihood Estimation} (UCMLE).

\begin{problem}[UCMLE]\label{prob-ucmle}
Given a finite set $X\subset \IR^d$ and an integer $K\in \IN$, find a partition $\cC = \{C_1,\ldots,C_K\}$ of $X$ into $K$ disjoint subsets and $K$ spherical Gaussians with parameters $\theta = \{(\mu_k,\sigma^2_k)\}_{k=1}^K$ minimizing
\begin{align*}
\cL^{unif}_X(\theta, \cC)
&=\sum_{k=1}^K  \cL_{C_k}(\mu_k,\sigma^2_k) \\
&= \sum_{k=1}^K \frac{\abs{C_k}d}{2}\ln(2\pi \sigma_k^2) + \frac{1}{2\sigma_k^2} \left( \sum_{x\in C_k} \norm{x-\mu_k}^2\right)  \ .
\end{align*}
We denote the minimal value by $OPT_{unif}(X,K)$.
\end{problem}

\begin{corollary}
Let $X\subset \IR^d$, $K\in \IN$, and $\delta,\epsilon > 0$.
Let $X=\dot\bigcup_{k=1}^K C_k$ be a well-defined solution for the UCMLE problem.
There is an algorithm that computes $K$ spherical Gaussians 
$\theta = \{(\tmu_k,\tsigma^2_k)\}_{k=1}^K$
such that with probability at least $1-\delta$
\[ \cL^{unif}_{X}((\tmu_i,\tsigma_i^2)_{i\OneTo{K}}) \leq  (1+\epsilon) OPT_{unif}(X,K)\ . \]
The running time of the algorithm is bounded by
\[ \abs{X}\,d\,\log(1/\delta)\,2^{\bigO(K/\epsilon\cdot \log(K/\epsilon^2))}\, \left(\log(\log(\Delta^2))+1\right)^K \ ,\]
where $\Delta^2 = \max_{x,y\in X} \{ \norm{x-y}^2\}$.
\end{corollary}
\begin{proof}
 Use a grid search to obtain candidates for the variances, then apply the ABS algorithm.
\end{proof}

\bibliographystyle{alpha}

\end{document}